\documentclass{article}

\usepackage{microtype}
\usepackage{graphicx}
\usepackage{subfigure}
\usepackage{booktabs} 
\usepackage{amsfonts}
\usepackage{amsmath}
\usepackage{amsthm}
\usepackage{thmtools, thm-restate}
\usepackage{xcolor}
\usepackage{tikz}
\usepackage{bm}
\usepackage{multirow}

\newcommand{\sam}[1]{\textbf{\textcolor{blue}{#1}}}

\usepackage{hyperref}

\newtheorem*{theorem-non}{Theorem}
\newtheorem{lemma}{Lemma}
\newtheorem{corollary}{Corollary}

\newcommand{\argmax}{\mathop{\mathrm{argmax}}}   

\usepackage[accepted]{icml2021}

\icmltitlerunning{Off-Belief Learning}

\begin{document}

\twocolumn[
\icmltitle{Off-Belief Learning}



\icmlsetsymbol{equal}{*}

\begin{icmlauthorlist}
\icmlauthor{Hengyuan Hu}{fb}
\icmlauthor{Adam Lerer}{fb}
\icmlauthor{Brandon Cui}{fb}
\icmlauthor{David Wu}{fb} 
\icmlauthor{Luis Pineda}{fb}
\icmlauthor{Noam Brown}{fb}
\icmlauthor{Jakob Foerster}{fb}
\end{icmlauthorlist}

\icmlaffiliation{fb}{Facebook AI Research}

\icmlcorrespondingauthor{Hengyuan Hu}{hengyuan@fb.com}
\icmlcorrespondingauthor{Jakob Foerster}{jnf@fb.com}

\icmlkeywords{Machine Learning, ICML}

\vskip 0.3in
]



\printAffiliationsAndNotice{}  

\begin{abstract}
The standard problem setting in Dec-POMDPs is self-play, where the goal is to find a set of policies that play optimally together. Policies learned through self-play may adopt arbitrary conventions and implicitly rely on multi-step reasoning based on fragile assumptions about other agents' actions and thus fail when paired with humans or independently trained agents at test time.
To address this, we present \emph{off-belief learning} (OBL). At each timestep OBL agents follow a policy $\pi_1$ that is optimized assuming \emph{past actions} were taken by a given, fixed policy ($\pi_0$), but assuming that \emph{future actions} will be taken by $\pi_1$. When $\pi_0$ is uniform random, OBL converges to an optimal policy that does not rely on inferences based on other agents' behavior (an optimal \emph{grounded} policy).
OBL can be iterated in a hierarchy, where the optimal policy from one level becomes the input to the next, thereby introducing multi-level cognitive reasoning in a controlled manner. Unlike existing approaches, which may converge to any equilibrium policy, OBL converges to a unique policy, making it suitable for zero-shot coordination (ZSC).
OBL can be scaled to high-dimensional settings with a \emph{fictitious transition} mechanism and shows strong performance in both a toy-setting and the benchmark human-AI \& ZSC problem Hanabi.
\end{abstract}

\section{Introduction}
\label{sec:intro}

An important goal of multi-agent reinforcement learning (MARL) research is to develop AI systems that can coordinate with novel partners, such as humans or other artificial agents, unseen at training time.
However, the standard problem setting in cooperative MARL is \emph{self-play} (SP), where a team of agents is trained to find near optimal joint policies. Crucially, SP assumes that at test time all parties follow this policy faithfully, without considering the performance when paired with other agents.
As a result, strong joint policies for SP often rely on efficient, yet arbitrary conventions
~\cite{foerster2019bayesian}.
%

Although SP agents achieve high scores among themselves, they often perform very poorly when paired with other, independently-trained agents or with humans because they cannot decipher the other agents' signals and intentions, and vice versa~\cite{carroll2019utility, bard2020hanabi, hu2020other}. Such agents fail because they rely too heavily on \emph{conventions} and other inferences that finely depend on their partner's policy, rather than on grounded information.


For example, suppose in the board game Hanabi (detailed game rules in Section~\ref{sec:result}) our partner reveals the color of a particular card, e.g. green. We now know that the card is green. This is \emph{grounded} information. But we may also infer other information. If we have agreed on a \emph{convention} to only reveal the color of a card if it is playable, we can infer that the green card must also be playable. So under this shared convention in this situation ``green'' also means ``playable'', not simply ``green''.  

Conventions may also be iterated to higher levels of \emph{cognitive depth}, where agents build more complex conventions on top of the baseline expectation that others are following earlier conventions.
For example, \emph{not} playing a card that was hinted to be playable could be used as a signal for an even higher priority action.

Agents trained via self-play tend to learn sets of conventions where the meaning of each action in each situation can vary unpredictably between different training runs. These conventions can be of arbitrarily high cognitive depth, which each agent expects the other to understand perfectly. 
By contrast, humans often limit themselves to a cognitive depth of only a few levels~\cite{camerer2003}, and fewer when other agents are anticipated to be less experienced~\cite{agranov2012449}, across a wide variety of competitive and collaborative games. When collaborating with unknown partners for the first time, humans may still successfully communicate by simple grounded information sharing. They can avoid making fragile inferences based on their partner's unknown policy, while still planning their own actions and assuming that other players will respond ``reasonably''. This is a behavior that current methods cannot replicate. 

In this paper, we present \emph{off-belief learning} (OBL), a MARL algorithm that addresses this gap in prior methods by controlling the cognitive reasoning depth, preventing the use of arbitrary conventions, while converging to optimal behavior under this constraint. Given a starting policy $\pi_0$, OBL agents follow a policy $\pi_1$ that, at each step, acts optimally assuming that all \emph{past actions} were taken by $\pi_0$ and all \emph{future actions} will be taken by $\pi_1$. Note that this includes accounting for the fact that future actions from $\pi_1$ will again be interpreted as if they were chosen by $\pi_0$. OBL can also be applied iteratively to obtain a hierarchy of policies that converges to a unique solution, making it a natural candidate for solving zero-shot coordination (ZSC) problems, where the goal is to maximize the test-time performance between polices from independent training runs of the same algorithm (cross-play).



Our work is motivated by the question ``Can an agent act optimally, and expect other agents to do the same, \emph{without} assuming any shared conventions?''. In our work we formalize and answer this question affirmatively with OBL.



In particular, if $\pi_0$ is a uniform random policy, then under OBL $\pi_1$ is an optimal \emph{grounded} policy, meaning that it plays the game relying only on the \emph{grounded information}, without assuming any conventions.
This is because interpreting actions as coming from $\pi_0$ results in $\pi_1$ conditioning only on grounded information revealed by the environment due to actions taken (a card revealed to be green is green, even if interpreted as if it were revealed by a random agent), but no further information about any agent's private state, since regardless of private state, a uniform random agent takes all actions with equal probability.
Crucially, since $\pi_1$ will act upon this grounded information optimally at future time steps, $\pi_1$ will still learn optimal \emph{grounded signaling}, i.e. taking actions specifically to reveal information other agents can use to take better actions.


In this work, we first characterize OBL theoretically. We prove that OBL converges to a unique policy, i.e. given the same $\pi_0$, OBL produces the the same $\pi_1$ across independent training runs. We then prove that OBL is a policy improvement operator under certain conditions. We further show that when $\pi_0$ doesn't depend on private information, OBL will converge to an optimal grounded policy.
We then evaluate OBL in both a toy setting and Hanabi. In the toy setting, we demonstrate that OBL learns an optimal grounded policy while other existing methods such as SP and Cognitive Hierarchies do not. In Hanabi, OBL finds fully-grounded policies that reach a score of 20.92 in SP without relying on conventions, an important data point that tells us how well we can perform in this benchmark without conventions. OBL also performs strongly in zero-shot coordination settings, achieving significantly higher cross-play scores than previous methods, particularly when iterated. Finally we find it also plays well in ad-hoc team play~\cite{stone2010ad} when paired with an agent trained by imitation learning from human data and two distinct agents trained with reinforcement learning, all of which are unseen during training time.

\section{Related Work}
\label{sec:related_work}
This work is inspired by the notion of Cognitive Hierarchies (CH)~\cite{camerer2004cognitive} and k-level reasoning. CH is a framework for decision-making in multi-agent settings. Like OBL, CH starts with an `uninformative' level-0 policy and constructs a hierarchy of policies that are optimal assuming that other agents play the policies at lower levels. 
However, in contrast to OBL, CH does not differentiate between reasoning about other agents' \emph{past behavior} and expected \emph{future behavior}. In other words, assuming that $\pi_0$ is a uniform random policy, $\pi_1$ under CH is simply a \emph{best response} to a random policy. While this seems like a reasonable default, it has some rather drastic consequences. In particular, unlike OBL, it is impossible for CH to learn optimal grounded policies in many settings. 
The intuition here is simple: Only the first CH level, the best response to a random agent, interprets actions as entirely grounded, \emph{i.e.}, conditioning only on observable information. However, the first level also assumes that the future moves of the partner are entirely random, which means it cannot learn strategies that require the partner to cooperate. One class of behavior that is difficult or impossible to learn is \emph{grounded signaling}, where an agent can reveal information to another agent through a costly action, as we show in Section~\ref{sec:toy_example}. 

The striking difference between OBL and CH can also be illustrated in fully observable turn-based settings. While OBL learns the optimal policy at the first level, CH in general needs as many levels as the length of the episode to converge to an optimal policy. Another difference is the role of $\pi_0$: While in CH the exact $\pi_0$ can change the final policy, in OBL any policy $\pi_0$ that ignores the private observation will result in the same, perfectly uninformative belief and lead to the same OBL policy $\pi_1$.

OBL is conceptually related to the Rational Speech Acts (RSA) framework, which has been able to model and explain a large amount of human communication behavior~\cite{frank2012predicting}. RSA assumes a speaker-listener setting with grounded hints and starts out with a literal listener (LL) that only considers the grounded information provided. Like CH, RSA then introduces a hierarchy of speakers and listeners, each level defined via Bayesian reasoning (\emph{i.e.}, a best response) given the level below.  
OBL allows us to train an analogous hierarchy and thus introduce pragmatic reasoning in a controlled fashion.  However, while RSA is mostly a conceptual framework that is focused on simple communication tasks, OBL is designed to deal with high-dimensional, complex Dec-POMDPs in which agents have to both act and communicate through their actions. 

There is a large body of work on coordination in behavioral game theory. One of the most well-known frameworks is focal points~\cite{schelling1980strategy}, \emph{i.e.}, common knowledge labels in the environment that allow agents to coordinate on a joint action. In contrast to this line of work we address the \emph{zero-shot coordination} setting~\cite{hu2020other}, where the structure of the Dec-POMDP itself needs to be used for coordination and players specifically cannot coordinate on labels for states, actions and observations. 

In this setting \emph{Other-play} (OP)~\cite{hu2020other} has been proposed as a method to prevent agents from learning equivalent but mutually incompatible policies across independent training runs.
OP accomplishes this by enforcing equivariance of the policies under the symmetries of the Dec-POMDP, which must be provided as an input to the algorithm. While OP prevents agents from breaking symmetries of the environment to encode conventions, it does not prevent all arbitrary convention formation as not all of them are symmetry-breaking. 

By starting from grounded information and deriving conventions in a controlled fashion, OBL avoids both symmetry-breaking and arbitrary convention formation when run with an appropriate (e.g. random) $\pi_0$. Unlike OP it does not require access to the symmetries. 

\section{Background}
\label{sec:background}

\textbf{Dec-POMDPs} We consider decentralized partially observable Markov Decision Processes (Dec-POMDPs) \cite{nair2003taming} $G$ with state $s$, action $a$, observation function $\Omega^i(s)$ for each player $i$ and transition function $\mathcal{T}(s, a)$. Players receive a common reward $R(s, a)$. We denote the historical trajectory as $\tau =(s_1, a_1, ..., a_{t-1}, s_t)$ and \textit{action-observation history} (AOH) for player $i$ as $\tau^i = (\Omega^i(s_1), a_1, ..., a_{t-1}, \Omega^i(s_t))$, which encodes the trajectory from player $i$'s point of view. A policy $\pi^i(a|\tau_t^i)$ for player $i$ takes as input an AOH and outputs a distribution over actions. We denote the joint policy as $\pi$. We remark that while the value $V^\pi(\tau)$ of policy $\pi$ from any state $s$ is well-defined, the value of playing $\pi$ from an AOH $\tau^i$ is not well-defined without specifying what policy was played \textit{leading} to $\tau^i$, because it affects the distribution of world states at $\tau^i$. The distribution of world states is often referred to as a belief $\mathcal{B}_{\pi}(\tau|\tau^i)=P(\tau|\tau^i,\pi)$. 


In this work, we restrict ourselves to \textit{bounded-length, turn-based Dec-POMDPs}. Formally, we assume that $G$ reaches a terminal state after at most $t_{max}$ steps, and that only a single agent acts at each state $s$ (other agents have a null action). See Section \ref{sec:turn_based} for a discussion of turn-based vs. simultaneous-action environments. 



\textbf{Deep MARL} has been successfully applied in many Dec-POMDP settings~\citep{mackrl,foerster2019bayesian}. The specific OBL algorithm introduced later uses Recurrent Replay Distributed Deep Q-Networks (R2D2)~\citep{r2d2} as its backbone. In deep Q-learning~\citep{dqn} the agent learns to predict the expected total return for each action given the AOH, $Q(\tau^{i}_{t}, a) = E_{\tau \sim P(\tau | \tau^{i})} R_{t}(\tau)$. $R_{t}(\tau) = \sum_{t'=t}^{\infty} \gamma^{(t' - t)} r_{t'}$ is the forward looking return from time $t$ where $r_{t'}$ is reward at step $t'$ with $a_t = a$ and $a_{t'} = \argmax_{a'} Q(\tau_{t'}, a')$ for $t' > t$ and $\gamma$ is an optional discount factor. The state-of-the-art R2D2 algorithm incorporates many modern best practices such as double-DQN~\citep{double-dqn}, dueling network architecture~\citep{dueling-dqn}, prioritized experience replay~\citep{prioritized-replay}, distributed training setup with parallel running environments~\citep{apex} and recurrent neural network for handling partial observability. 

The straightforward way to apply deep Q-learning to Dec-POMDP settings is Independent Q-Learning (IQL)~\citep{tan93multi} where each agent treats other agents as part of the environment and learns an independent estimate of the expected return without taking other agents' actions into account in the bootstrap process. Many methods~\citep{vdn, qmix} have been proposed to learn joint Q-functions to take advantage of the centralized training and decentralized control structure in Dec-POMDPs. In this work, however, we use the IQL setup with shared neural network weights $\bm{\theta}$ for simplicity and note that any potential improvements to the RL algorithm used are orthogonal to our contribution.

\textbf{Zero-Shot Coordination.} The most common problem setting for learning in Dec-POMDPs is \emph{self-play} (SP) where a team of agents is trained and tested together. Optimal SP policies typically rely on \emph{arbitrary conventions}, which the entire team can jointly coordinate on during training. However, many real-world problems require agents to coordinate with other, unknown AI agents and humans at test time.
This desiderata was formalized as the \emph{Zero-Shot Coordination} (ZSC) setting by \citet{hu2020other}, where the goal is stated as finding algorithms that allow agents to coordinate with \emph{independently trained} agents at test time, a proxy for the the independent reasoning process in humans. ZSC immediately rules out \emph{arbitrary conventions} as optimal solutions and instead requires learning algorithms that produce robust and, ideally, unique solutions across multiple independent runs. 


\section{Off-Belief Learning}
\label{sec:method}
One of the big challenges in ZSC under partially observable settings is to determine how to \emph{interpret} the actions of other agents and how to select actions that will be interpretable to other agents. 
In the following we introduce OBL, a method that addresses this problem by learning a hierarchy of policies, with an optimal \emph{grounded} policy at the lowest level, which does not interpret other agents' actions at all.
We will discuss OBL both as a formal method with corresponding proofs and as a scalable algorithm based on a \emph{fictitious transition} mechanism applicable to the deep RL setting.

\subsection{The OBL Operator}
We first introduce the OBL operator that computes $\pi_1$ given any $\pi_0$. 
If a common knowledge policy $\pi_0$ is played by all agents up to $\tau^i$, then agent $i$ can compute a belief distribution $\mathcal{B}_{\pi_0}(\tau|\tau^i)=P(\tau|\tau^i,\pi_0)$ 
conditional on its AOH. This belief distribution fully characterizes the effect of the history on the current state. Notice that the return for (all players) playing a `counterfactual' policy $\pi_0$ 
to $\tau^i$ and $\pi_1$ thereafter, which we denote $V^{\pi_0 \to \pi_1}(\tau^i)$, is precisely the expected return for all players playing according to $\pi_1$, starting from a trajectory $\tau \sim \mathcal{B}_{\pi_0}(\tau^i)$. Therefore we define a counterfactual value function as follows:
\begin{equation}
    V^{\pi_0 \to \pi_1}(\tau^i) = \mathbb{E}_{\tau \sim \mathcal{B}_{\pi_0}(\tau^i)}\left[ V^{\pi_1}(\tau) \right].
    \label{eq:V_counterfactual}
\end{equation}
We can similarly define counterfactual $Q$ values as
\begin{multline}
    Q^{\pi_0 \to \pi_1}(a|\tau_t^i) = \sum_{\tau_t,\tau_{t+1}}\mathcal{B}_{\pi_0}(\tau_t|\tau_t^i) \big[ R(s_t, a) \\
    \hspace{2cm} + \mathcal{T}(\tau_{t+1}|\tau_t) V^{\pi_1}(\tau_{t+1})\big].
    \label{eq:Q_counterfactual}
\end{multline}

Then \textit{OBL} operator is defined to be the operator that maps an initial policy $\pi_0$ to a new policy $\pi_1$ as follows:

\begin{equation}
    \label{eq:pi1}
    \pi_1(a|\tau^i)=\frac{\exp (Q^{\pi_0\to \pi_1}(a|\tau^i)/T) }{\sum_{a'} \exp (Q^{\pi_0\to \pi_1}(a'|\tau^i)/T)}
\end{equation}
for some temperature hyperparameter $T$.

\subsection{Properties of Off-Belief Learning}

\begin{restatable}{theorem}{thmunique}
\label{thm:unique}
For any $T>0$ and starting policy $\pi_0$, OBL computes a unique policy $\pi_1$.
\end{restatable}
Proof in the Appendix~\ref{app:proofs}. ``Unique'' means that that\sam{,} once $T$ is fixed, $\pi_1$ is a well-defined deterministic function of $\pi_0$. This means that, unlike most MARL methods, OBL always converges to the same answer in principle, regardless of random initialization or other hyperparameters.
\vspace{0.15cm}
\begin{corollary}
For any $T>0$ and starting policy $\pi_0$, $N$ agents independently computing OBL joint policy $\pi_1 = \{\pi_1^i\}$ and playing their parts of it ($\pi_1^i$) will achieve the same return as if they had computed a centralized $\pi_1$ policy (zero-shot coordination).
\label{cor:pi1}
\end{corollary}
\begin{proof}
Since all $\pi_1$ are identical, this follows trivially.
\end{proof}

Theorem $\ref{thm:unique}$ is trivial in a single-agent context, but illustrates a substantial departure from traditional multi-agent learning rules, under which independently computed policies for each agent are typically \textit{not} unique or compatible due to the presence of multiple equilibria, a particularly severe example of which is the formation of `arbitrary' conventions for communication in a game like Hanabi.

\begin{restatable}{theorem}{thmpolicyimprovement}
For every policy $\pi_1$ generated by OBL from $\pi_0$, $J(\pi_1) \geq J(\pi_0) - t_{max}T/e$, i.e. OBL is a policy improvement operator except for a term that vanishes as $T \to 0$.
\label{thm:policy_improvement}
\end{restatable}

Proof in the Appendix~\ref{app:proofs}. Theorem~\ref{thm:policy_improvement} implies that self-play performance will gradually increase if we apply OBL iteratively. Together with Corollary~\ref{cor:pi1}, it also guarantees that zero-shot coordination performance, i.e. cross-play between independent training runs, will improve at the same pace.

Notably, unlike standard MARL, the fixed points of the OBL learning rule \textit{are not} guaranteed to be equilibria of the game.
However, we have the theorem:
\begin{restatable}{theorem}{thmequilibrium}
    If repeated application of the OBL policy improvement operator converges to a fixed point policy $\pi$, then $\pi$ is an $\epsilon$-subgame perfect equilibrium of the Dec-POMDP, where $\epsilon=t_{max}T/e$.
\end{restatable}
Proof in the Appendix~\ref{app:proofs}.

\subsection{Optimal Grounded Policies}
\label{sec:optimal_grounded_policy}
In multi-agent cooperative settings with imperfect information, the optimal policy for an agent depends not only on assumptions about the future policy, but also on assumptions about other agents' policies -- i.e. what do those agents' prior actions say about their private information? Making wrong assumptions is a source of coordination failure and we thus may wish to consider grounded policies that avoid reasoning about partners' actions altogether. 

To do so, we define the \textit{grounded belief} $\mathcal{B}_G$ as a modified belief that conditions only on the observations but not the partner actions, i.e. $$\mathcal{B}_G(\tau|\tau^i) = \frac{P(\tau) \prod_t{  P( o_t^i | \tau ) } }{ \sum_{\tau'} P(\tau') \prod_t{  P( o_t^i | \tau' ) } }.$$ Then, we can define an \textit{optimal grounded policy} $\pi_G$ to be any policy that at each AOH  $\tau^i$ plays an action that maximizes expected reward assuming the state distribution at $\tau^i$ is drawn from the grounded belief, and assuming that $\pi_G$ is played thereafter.
\begin{restatable}{theorem}{thmgrounded}
\label{thm:grounded}
Application of OBL to any constant policy $\pi_0(a|\tau^i)=f(a)$ - or in fact any policy that only conditions on public state - yields an optimal grounded policy in the limit as $T \to 0$ .
\end{restatable}
Proof in the Appendix~\ref{app:proofs}.

\subsection{Iterated Application of OBL}
\label{sec:iteratedobl}
While an optimal grounded policy can be an effective baseline in some settings, clearly there are many settings where reasoning about the partner's behavior improves performance. In humans this ability is commonly described as \emph{theory of mind} and a variety of experimental studies \cite{camerer2003,agranov2012449,kleiman2016coordinate} have shown that humans typically carry out a limited number of such reasoning steps, depending on the exact situation, even in zero-shot settings.

Iterating OBL is simple. When the OBL operator is applied to a random policy $\pi_0$, it computes an optimal grounded policy $\pi_1$, which we also refer to as OBL ``level 1''. We may compute OBL level 2 similarly by using OBL level 1 as $\pi_0$ and compute OBL level 3 by replacing $\pi_0$ with OBL level 2, and so on.

Level 2 acts optimally while using beliefs induced by level 1 to interpret partners' actions. Different from OBL level 1, a level 2 agent will assume that other agents will take positive-utility actions instead of purely random actions and share useful factual information to others. For example, if it sees another agent avoid an action that would likely be beneficial, it will infer that the other agent probably has private information that the action is not so beneficial.

Since all independently trained agents use the same belief (modulo convergence errors), this training procedure is helpful in ZSC, unlike self-play where agents may learn arbitrary beliefs and signaling. Iterated OBL beliefs also appear to be a better match for human behavior than any prior work. We also show in Section~\ref{sec:result} that while OBL level 1 already cooperates better with a human-like policy than any prior method, level 2 and higher cooperate even better with it and other tested agents.

The iterated application of OBL allows us to directly control for the number of cognitive reasoning loops to be carried out by our agent. Different from the cognitive hierarchies framework, at each level the agent plays optimally assuming optimal game play in the future, rather than assuming the partners are fixed and play one level below. Therefore it is possible to optimize the overall \emph{quality} of the policy independently of the depth of cognitive reasoning in OBL but not in cognitive hierarchy because we cannot reach the optimal policy at each level.

\begin{figure*}[t]
\centering
\includegraphics[width=0.45\linewidth]{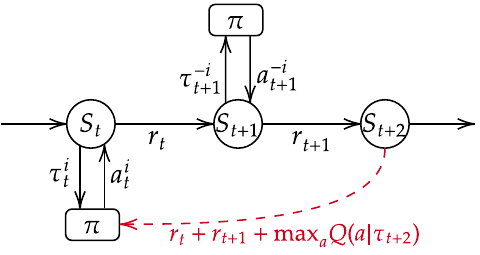}
\includegraphics[width=0.45\linewidth]{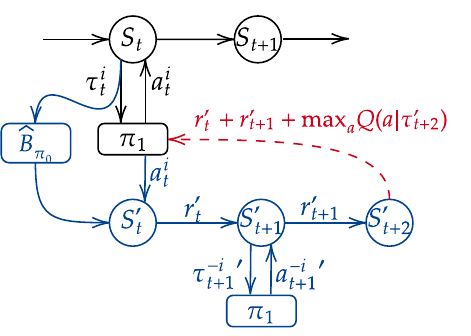}
\caption{\small Comparison of independent Q learning (left) with LB-OBL (right). Superscripts $i$ and $-i$ mean player $i$ and its partner $-i$. In LB-OBL, the target value for an action $a^i_t$ is computed by first sampling a state from a belief model $\hat{B}_{\pi_0}$ that assumes $\pi_0$ has been played, and then simulating play (blue) from that state to $i$'s next turn using the current policy $\pi_1$.}
\label{fig:iql-fig}
\end{figure*}

\subsection{Simultaneous-action vs turn-based environments}
\label{sec:turn_based}

So far we have restricted ourselves to \textit{turn-based} games in which a single agent acts at each state. If two players acted simultaneously, then OBL couldn't guarantee convergence to a unique policy (Thm \ref{thm:unique}); for example, consider a one-step coordination game with payoff matrix $[[1, 0], [0, 1]]$ (which is a single-state, single-step Dec-POMDP). OBL is identical to self-play in a one step game, and could thus converge to either equilibrium.

This may seem like a strong restriction of OBL to turn-based games. However, any simultaneous-action game can be converted to an equivalent stage game by ordering the simultaneous actions into separate timesteps and only letting players observe the actions of others once the timesteps corresponding to the simultaneous actions have completed. An OBL policy can be generated for this equivalent turn-based game. The ordering of players' actions solves the equilibrium selection problem, because the player who acts second assumes that the other player acted according to $\pi_0$. The player order we use to convert the simultaneous-action game to a turn-based one may change the OBL policy.

\subsection{Algorithms for Off-Belief Learning}
\label{sec:algo-obl}

Equation \ref{eq:pi1} immediately suggests a simple algorithm for computing an OBL policy in small tabular environments: compute $\mathcal{B}_{\pi_0}(\tau^i)$ for each AOH, and then compute $Q^{\pi_0\to\pi_1}(\tau^i)$ for each AOH in `backwards' topological order. However, such approaches for POMDPs are intractable for all but the smallest size problems.

In order to apply value iteration methods, we can write the Bellman equation for $Q^{\pi_0 \to \pi_1}$ for each agent $i$ as follows:
\begin{equation}
\begin{split}
    Q^{\pi_0\to \pi_1}(a_t|\tau_t^i) = \mathbb{E}_{\tau_t, \tau_{t+k}} \Big[ \sum_{t'=t}^{t+k-1} R(\tau_{t'}, a_{t'}) + \\
    \sum_{a_{t+k}} \pi_1(a_{t+k}|\tau_{t+k}^i) Q^{\pi_0\to \pi_1}(a_{t+k}|\tau_{t+k}^i)\Big]
\end{split}
\label{eq:ob_bellman}
\end{equation}
where $\tau_{t+k}$ is the next history at which player $i$ acts, $\tau_t \sim \mathcal{B}_{\pi_0}(\tau^i_t)$, and $\tau_{t+k} \sim (\mathcal{T},\pi_1)$ denotes that $\tau_{t+k}$ is sampled from the distribution of successor states where all (other) players play according to $\pi_1$. 

Now, how can we perform the Bellman iteration in Eq. \ref{eq:ob_bellman}? One approach, which we will denote \emph{Q-OBL}, is to perform independent $Q$-learning where each player uses $\pi_0$ as the exploration policy to play until time $t$ and use $\pi_1$ afterwards. We can then train $\pi_1$ at time $t$ but not other timesteps.
This would guarantee that $\tau_t \sim \mathcal{B}_{\pi_0}(\tau_t^i)$. However in order to compute the value for playing $\pi_1$ in the future, we must simulate the transition $\tau_{t+k}\sim(\mathcal{T},\pi_1)$ by having the other players playing according to their current $\pi_1$. This is guaranteed to converge to the OBL policy (in the tabular case) as long as $\forall \tau^i\ \forall a\ \pi_0(a|\tau^i) > 0$, by the same inductive argument used in Theorems \ref{thm:unique} and \ref{thm:policy_improvement}, since $Q$ and $\pi_1$ at each AOH only depend on $Q$ and $\pi_1$ at successor AOHs.

\begin{figure*}[h]
\begin{center}
  \includegraphics[width=1.0\textwidth]{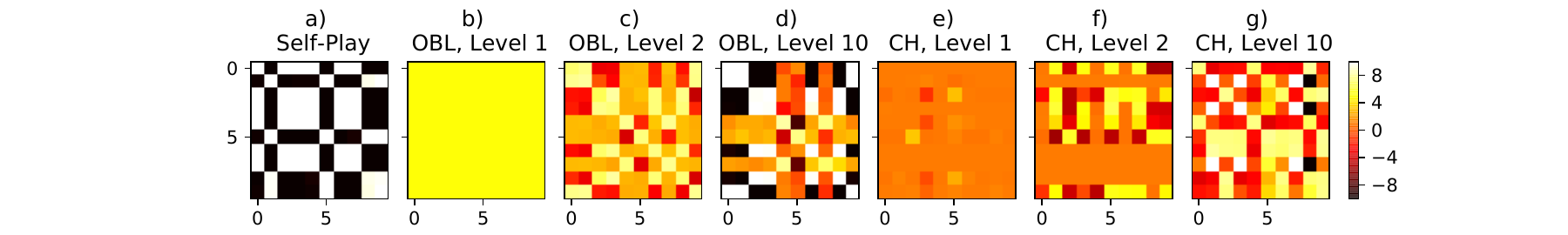}
      \vspace{-18pt}
  \caption{Cross-play matrix for each of the different algorithms in the toy communication game. 10 agents were independently trained with each method; each cell (i, j) of the matrix represents the average score of the $i$th and $j$th agent. OBL with a depth of 1 \textbf{(b)} achieves a higher score than CH \textbf{(e)} by learning optimal grounded communication, while avoiding arbitrary handshakes that lead to the poor XP performance of SP \textbf{(a)} or multi-level reasoning \textbf{(d, g)}.}
      \vspace{-8pt}
  \label{figs:simple_game_results}
\end{center}
\end{figure*}

In practice, this approach has a downside: the states reached by $\pi_1$ may be reached with very low probability under $\pi_0$,\footnote{The bound is $P(\tau|\pi_0)/P(\tau|\pi_1) \geq (\min_a P(a))^t$} so this procedure will have low sample efficiency. Instead, we first learn an \textit{approximate belief} $\hat{\mathcal{B}}_{\pi_0}$ that takes an AOH $\tau^i$ and can sample a trajectory from an approximation of $P(\tau|\tau^i,\pi_0)$. We compute this belief model following the procedure described in \citet{hu2021learned}. We then perform $Q$-learning with a modified rule for computing the target value for $Q(a|\tau_t^i)$: we re-sample a \textit{new} $\tau'$ from $\hat{B}_{\pi_0}(\tau_t^i)$, simulate a transition to $\tau_{t+1}^{i'}$ with other agents play their policy $\pi_1$, and use $\max_a Q(a|\tau_{t+1}^i)$ as the target value. We refer to this variant as \textit{Learned-belief OBL} (LB-OBL).

An illustration comparing the LB-OBL algorithm with independent Q-learning (IQL) is shown in Figure~\ref{fig:iql-fig}. In both cases, we assume a two player setting with shared Q-function and 2-step Q-learning, i.e. $k=2$ in Equation~\ref{eq:ob_bellman}. The active player at time $t$ is $i$ and the partner is $-i$. In IQL, each player simply observes the world at each time-step and takes turns to act. The target is computed on the actual trajectory with $r_t$, $r_{t+1}$ and $\max_{a}Q(a|\tau_{t+2})$ using target network $Q^{i}$. By contrast, LB-OBL involves only \emph{fictitious transitions}. For the active player $i$, we first decide the action $a_t^{i}$ given AOH $\tau_t^i$. However, in addition to applying $a_t^{i}$ to the actual environment state $S_t$, we also apply $a_t^{i}$ to a fictitious state $S^{\prime}_{t}$ sampled from the learned belief model $B_i$ and forward the fictitious environment to $S^{\prime}_{t+1}$. Note that the action $a_t^i$ can be directly applied to the fictitious state because the observation from the fictitious state is the same as that of the real state. Next, we need to evaluate the partner's policy to produce a fictitious action $a^{-i\prime}_{t+1}$ on $\tau^{-i\prime}_{t+1} = [\tau^{-i}; a^{i}_t, \Omega^{-i}(S^{\prime}_{t+1})]$. The learning target in LB-OBL is the sum of \emph{fictitious} rewards $r'_{t}$, $r'_{t+1}$ and the \emph{fictitious} bootstrapped value $\max_{a}Q(a|\tau'_{t+2})$. We also note that it is not possible to simply train a Q-learning agent that takes the \emph{grounded beliefs} as an input. The agents would still develop conventions and the \emph{grounded belief} would simply lose its semantic meaning. 

\section{Experiments in a Toy Environment}
\label{sec:toy_example}

\begin{figure}[t]
    \centering
    \includegraphics[width=0.75\columnwidth]{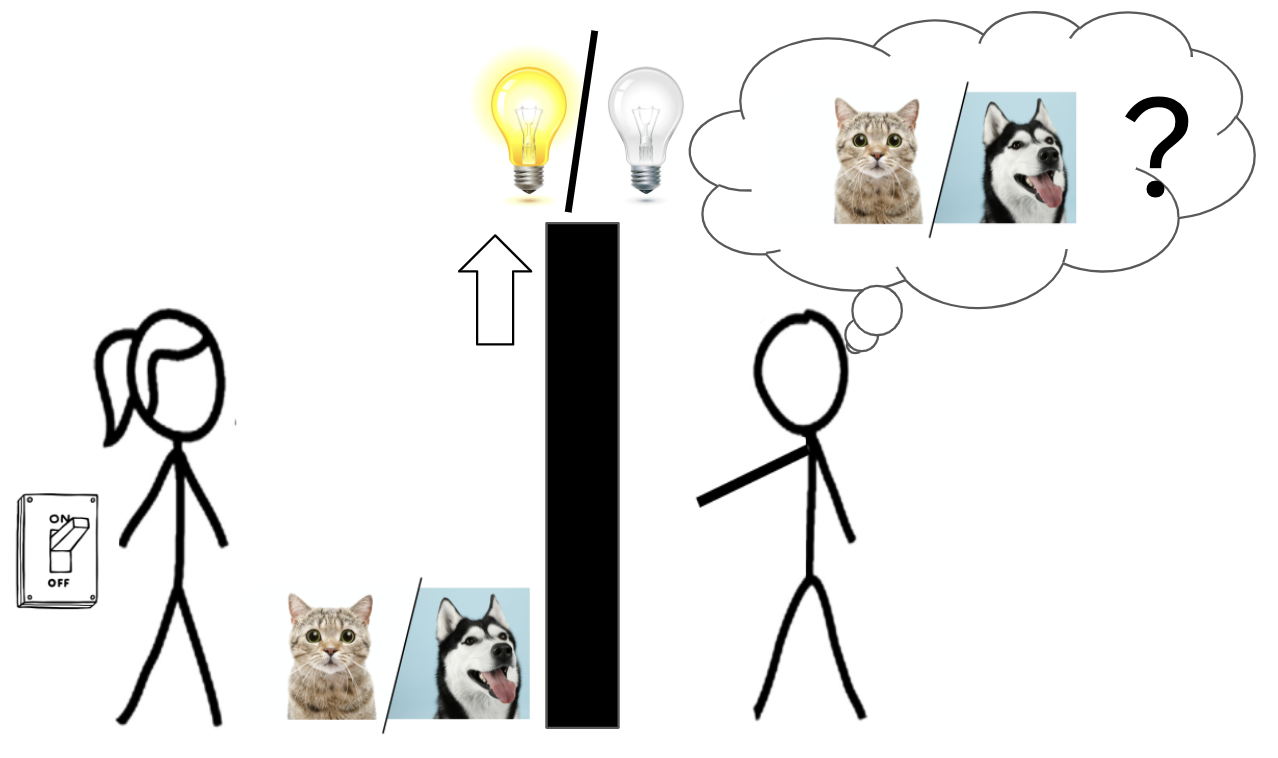}
    \caption{\small Toy cooperative communication game: Alice (left) observes the pet and can either signal to Bob by turning on the light-bulb, pay 5 to remove the barrier, so that Bob can see the pet, or bail out for a reward of 1. Bob then needs to guess the identity of the pet or can bail out for a reward of 0.5. }
    \vspace{-18pt}
    \label{figs:simple_game}
\end{figure}

We first test OBL in a simple fully cooperative toy environment that includes both a cheap-talk channel and a grounded communication channel:
As shown in Figure~\ref{figs:simple_game}, Alice observes a random binary variable, $\text{pet} \in \{ \text{cat}, \text{dog} \}$, and has four possible actions. She can either turn on the light bulb, bail out of the game with fixed reward of 1 or remove a barrier to let Bob see the pet, which will cost them 5. After Alice takes her action, Bob observes the outcome and has three options. He can either bail for a fixed reward of 0.5 or guess the identity of the pet. If the guess is correct, they obtain an additional $10$, otherwise they lose $10$. 

\textbf{Results.} We run three different algorithms on this simple toy game: SP, OBL and cognitive hierarchies (CH). 
Figure~\ref{figs:simple_game_results} shows the cross-play (XP) results for 10 independent runs for each of the methods as a matrix, where the entry $i,j$ is the reward obtained from pairing the first agent (Alice) from run $i$ with the second agent (Bob) from run $j$. As can be seen from the checkerboard-like pattern, the SP agents learn to use the cheap-talk channel, obtaining $+10$ in about half the cases and $-10$ in the others. Clearly, these agents have developed an arbitrary handshake where ``on'' / ``off'' correspond to dog / cat or vice versa, each incompatible with the other half of the agents that made the opposite choice. 
In contrast, OBL learns the optimal \emph{grounded policy} at the first level, choosing to remove the barrier which results in a reward of $+5$, both in XP and SP.
The CH agents at level 1 learn to simply bail out of the game, since there is no way to win when playing with a random agent. They also fail to discover the grounded solution in between and proceed straight to an arbitrary handshake for higher levels, as is indicated by the $+10 / -10$ rewards.

At higher levels, our OBL implementation starts to diverge between different runs. Although theoretically the OBL policy is unique, some environment states do not occur very often and the noise is amplified higher up. This could be addressed with higher temperature or increased regularization if desired. All results can be reproduced in our notebook implementation \url{https://bit.ly/3cIia6z}.

\section{Hanabi Experiments}
\label{sec:result}

\begin{table*}[]
    \centering
    \begin{tabular}{l c c c c c}
    \toprule
    \multirow{2}{*}{Method} & \multirow{2}{*}{Self-Play} & \multirow{2}{*}{Cross-Play} &  w/ OP  & w/ OP & \multirow{2}{*}{w/ Clone Bot} \\
    & & &(Rank Bot) & (Color Bot) & \\
    \midrule
    SAD(*)     & 23.97 $\pm$ 0.04 &  2.52 $\pm$ 0.34 &  7.29 $\pm$ 1.40 &  0.18 $\pm$ 0.06 &  0.83 $\pm$ 0.31 \\
    OP & 24.14 $\pm$ 0.03 &	21.77 $\pm$ 0.68 & 22.81 $\pm$ 0.87 &  4.05 $\pm$ 0.37 &  8.55 $\pm$ 0.48 \\
    K-Level    & 16.97 $\pm$ 1.19 & 17.17 $\pm$ 0.98 & 14.80 $\pm$ 1.77 & 12.36 $\pm$ 1.44 & 13.03 $\pm$ 1.91 \\
    \midrule
    OBL (level 1) & 20.92 $\pm$ 0.07 & 20.85 $\pm$ 0.03 & 10.83 $\pm$ 0.26 & 13.21 $\pm$ 0.34 & 13.56 $\pm$ 0.15 \\
    OBL (level 2) & 23.41 $\pm$ 0.03 & 23.24 $\pm$ 0.03 & 15.99 $\pm$ 0.30 & 18.74 $\pm$ 0.46 & 16.03 $\pm$ 0.14 \\
    OBL (level 3) & 23.93 $\pm$ 0.01 & 23.68 $\pm$ 0.05	& 15.61 $\pm$ 0.34 & 20.68 $\pm$ 0.44 & 16.54 $\pm$ 0.28 \\
    OBL (level 4) & 24.10 $\pm$ 0.01 & 23.76 $\pm$ 0.06	& 14.46 $\pm$ 0.59 & 21.78 $\pm$ 0.42 & 16.76 $\pm$ 0.16 \\

    \bottomrule
    \end{tabular}
    \caption{Average scores in 2-player Hanabi for different pairings of agents. SP indicates play between an agent and itself. XP indicates play between agents from different independently-trained runs of the same algorithm. The other columns indicate play with agents from different algorithms. We evaluate each model on 5000 games and aggregate results of 5 independent training runs (seeds). (*) We use the 12 SAD agents obtained from~\cite{hu2020other} which uses different network architectures. Thus the numbers are not directly comparable.}
    \label{tab:tab:hanabi-result}
\end{table*}

We now test our methods in the more complex domain of Hanabi. Hanabi is a fully-cooperative multiplayer partially-observable board game. It has become a popular benchmark environment~\citep{bard2020hanabi} for MARL, theory of mind, and zero-shot coordination research.

Hanabi is a 2-5 player card game. The deck is composed of 50 cards, split among five different colors (suits) and ranks, with each color having three 1s, two 2s and 3s and 4s, and one 5. In a 2-player game, each player maintains a 5-card hand. Players can see their partner's hand but not their own. The goal of the team is to play one card of each rank in each color in order from 1 to 5. The team shares 8 hint tokens and 3 life tokens. Taking turns, each turn a player can play or discard a card in their hand, or spend a hint token to give a hint to their partner. Playing a card succeeds if it is the lowest-rank card in its color not yet played, otherwise it fails and loses a life token. Giving a hint consists of choosing a rank or a color that a partner's hand contains and indicating \emph{all} cards in the partner's hand sharing that color or rank. Discarding a card or successfully playing a 5 regains one hint token. The team's score is zero if all life tokens are lost, otherwise it is equal to the number of cards successfully played, giving a maximum possible score of 25.


\textbf{Experimental Setup.} In Hanabi, we implement LB-OBL on top of R2D2~\cite{r2d2}. Training consists of a large number of parallel environments that invoke a deep neural network to approximate the Q-function at each time step to generate trajectories, a prioritized experience replay buffer to store the trajectories, and a training loop that updates the neural network using samples from the replay buffer. The queries from different environments are dynamically batched together to run efficiently on GPUs~\citep{Espeholt2020SEED}. We introduce several novel engineering designs to make our training infrastructure significantly more efficient than previous ones in Hanabi, allowing us to reach new state-of-the-art results in SP. We refer to Appendix~\ref{app:details} for more details on the implementation and benchmarks.

Our policy ($Q$-function) is parameterized by a deep recurrent neural network $\bm{\theta}$. As illustrated in Figure~\ref{fig:iql-fig}, we should sample trajectory $\tau$ from replay buffer and update the neural network with TD-error:
\begin{equation}
\mathcal{L}(\bm{\theta} | \tau) = \frac{1}{2} \sum_{t=1}^{T} [r'_t + r'_{t+1} + \max_a Q_{\hat{\bm \theta}}(a|\tau'_{t+2}) - Q_{\bm \theta}(a_{t} | \tau_{t})]^2
\end{equation}
where $Q_{\hat{\bm \theta}}$ is the target network that uses a slightly outdated version of $\bm{\theta}$ for stability reasons. In normal $Q$-learning with RNNs, the target $Q_{\hat{\bm \theta}}(a_t | \tau_t)$ for $t = 1, 2, \dots, T$ can be re-computed efficiently by passing the sequence $\tau$ to the RNN at once. However, this is no longer feasible in OBL as each $\tau'_{t}$ contains unique fictitious transitions. To solve this problem, we pre-compute the target $G'_t = r'_t + r'_{t+1} + \max_a Q_{\hat{\bm \theta}}(a|\tau'_{t+2})$ during rollouts and store the sequence of $\{{G'_t}\}$ along with the real trajectory $\tau$ into the replay buffer. Although this might cause stability issues if the stored target was computed using a very old target network, in practice we find it to be fairly stable because the replay buffer gets refreshed quickly thanks to the fast trajectory generation speed of our training infrastructure. On average, a trajectory is used 4-5 times before it gets evicted from the buffer. We use the public-LSTM architecture~\cite{hu2021learned} where the LSTM sub-network conditions only on the public part of the observation so that the input to the LSTM is consistent across normal transitions and fictitious transitions. 

The belief model is trained to predict the player's hand given $\tau^i$, which is the only missing information needed to produce fictitious game states in Hanabi. It takes as input $\tau^i$ and predicts each of the player's cards auto-regressively from the oldest to the newest. The belief model is an RNN trained via supervised learning:
\begin{equation}
    \mathcal{L}(\bm{h}|\tau^i_t) = -\sum_{k=1}^{n} \log p(h_k | \tau^i_t, h_{1:k-1}),
\end{equation}
where $h_k$ is the $k$th card in hand and $n$ is the hand size. 

To compare the zero-shot coordination performance of OBL against previous methods, we train three categories of policies with existing methods to serve as the unseen partner for OBL models. ``Rank-Bot" is trained with Other-Play (OP). Empirically, OP in Hanabi consistently converges to using rank-based conventions to communicate to a partner to play a card, hence the name. ``Color Bot" is roughly the color equivalent of ``Rank Bot", induced by adding extra reward for hinting color at the early stage of training. ``Clone Bot" is trained with supervised learning on human game data collected from an online board game platform. 

Similar to the toy example, we implemented a variant of Cognitive Hierarchies (CH) called k-level reasoning as one of our baselines. The details of neural network design, hyper-parameters and computation cost of OBL policies and belief models, as well as those of Rank, Color, Clone Bot and K-Level can be found in the Appendix~\ref{app:otherbots}. We will open source our code and all models. 

For all the methods and OBL levels shown in Table~\ref{tab:tab:hanabi-result}, we carry out 5 independent runs with different seeds, except for the 12 SAD models directly downloaded from the open-sourced repository of~\citet{hu2020other}.

\begin{figure}[t]
\centering
\includegraphics[width=\linewidth]{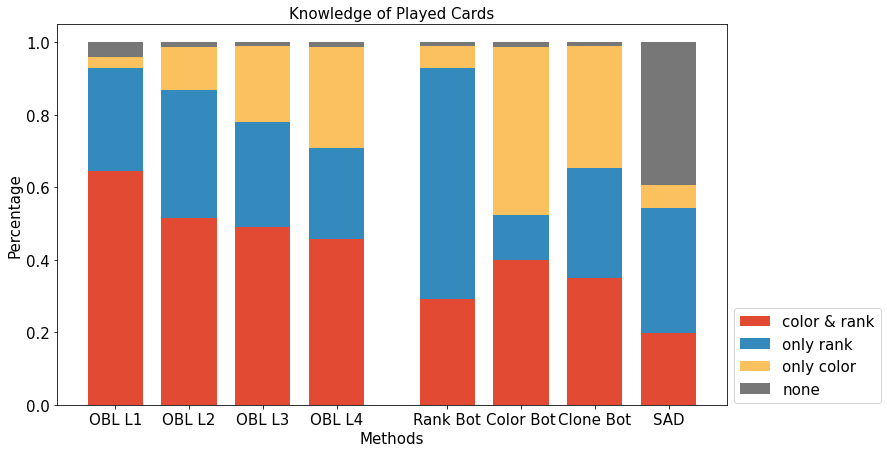}
\vspace{-15pt}
\caption{\small Grounded knowledge of cards played. OBL-1 mostly plays cards of known rank and color. Iterated OBL converges to an even balance of color and rank, similar to humans.}
\vspace{-10pt}
\label{fig:hanabi-card-knowlege}
\end{figure}

\textbf{Grounded Policies in Hanabi.} 
In this experiment, we empirically show that OBL can indeed learn surprisingly strong grounded policies in Hanabi. 

Although Theorem~\ref{thm:grounded} states that applying OBL on any constant policy $\pi_0$ should yield an optimal grounded policy $\pi_1$, in practice we want the trajectories generated by $\pi_0$ to be as diverse as possible to help train a good belief model. Hence a simple choice is to make $\pi_0$ a uniform-random policy. We first train a belief model $B_{\pi_0}$ to convergence and then use $B_{\pi_0}$ to train the OBL policy. We repeat this training procedure with five different seeds. The numerical results are shown in the \emph{OBL (level 1)} row of Table~\ref{tab:tab:hanabi-result}. We find that a policy with no conventions can achieve nearly 21 points in SP, which by itself is an important milestone for this popular benchmark environment. The XP score, which is computed by pairing agents with different seeds, is almost as high as the SP score, indicating that they indeed converge to nearly identical policies. The slight gap between XP and SP scores may come from the noise of belief and policy training. we note that despite being a grounded policy, OBL (level 1) already cooperates with human-like policies in Hanabi (proxied by CloneBot) better than any prior method. 

One way to quantitatively show whether a Hanabi policy is grounded is to look at how much the active player knows about a card before playing it. Playing a card unsuccessfully is costly, and therefore playing it without high confidence that it will succeed is usually suboptimal. Cards played can be categorized into four groups. 1) Both color and rank of the card are known. 2) Only color is known. 3) Only rank is known. 4) Neither is known. Here the known information can be revealed by hints or deduced from card counts. The result is shown in the \emph{OBL L1} column on Figure~\ref{fig:hanabi-card-knowlege}. 
More than 60\% of the cards played by this policy are completely known to the player, significantly higher than for other agents, which may infer playability based on arbitrary signals rather than grounded knowledge. Although roughly 28\% of the time the policy plays cards only knowing the number, inspection showed that in those cases the cards are safe to play regardless of color. 
These observations are consistent with the theoretical prediction that OBL should converge to the optimal grounded policy in Hanabi.

\textbf{OBL Hierarchy in Hanabi.} As described in Section~\ref{sec:iteratedobl}, we can repeat the OBL training process to get OBL level $i+1$ using level $i$ as the new $\pi_0$. In practice, we can speed up the training of next level belief and policy by loading weights from the previous level as a start point. 

We train three additional levels and the evaluation results are shown in Table~\ref{tab:tab:hanabi-result}. OBL level 2 infers beliefs over unseen cards assuming other agents are more likely to take good actions and to share factual information about useful cards the way OBL level 1 would, and so on for higher levels. Since other agents, even if they have unknown conventions, will also often act this way, these beliefs are sufficiently robust that OBL level 2 achieves even better scores than OBL level 1 with Clone Bot and with all other tested agents. 

For comparison, we show results of OP applied on the same network architecture as OBL. The OP agents are also trained with the auxiliary task from the original paper. We see from the table that both SP and XP of OBL gradually increase after each level. The final SP performance is similar to that of OP while the XP score is significantly higher, showing that OBL is a better method for zero-shot coordination. OBL also performs significantly better than OP when playing with a diverse set of unseen partners. Note that OP gets high scores with Rank Bot because Rank Bot is randomly chosen from the five OP training runs. Based on its performance with Clone Bot, it is probable that OBL would achieve much better performance than OP when paired with real humans. From Figure~\ref{fig:hanabi-card-knowlege}, we can see that as we progress through OBL levels, the agents play fewer grounded cards and start infer playability based on their partners' behaviors. Unlike SP where conventions are formed implicitly as each agent adapts to the random biases of its partner, OBL exhibits a consistent path of evolution that can be understood by reasoning through the behaviors of the lower level OBL policy. The final level of OBL is the most similar to Clone Bot in how it decides to play cards. 

\textbf{OBL in 3 Player Hanabi.} We also apply OBL to 3 player Hanabi and evaluate its performance in XP and playing with Clone Bot. All OBL levels perform better than OP in these two criteria. In XP, OBL (level 4) gets 23.02, which is significantly higher than the 17.36 of OP. When paired with Clone Bot, OP gets 12.18 while the highest score of OBL is 14.09 from level 2. The Clone Bot itself gets 17.46 in self-play. Due to the page limit the detailed results are available in Appendix~\ref{app:3player}.
\section{Conclusion and Future Work}
\label{sec:conclusion}
We present \emph{off-belief learning}, a new method that can train optimal \emph{grounded} policies, preventing agents from exchanging information through arbitrary conventions. When used in a hierarchy, each level adds one step of reasoning over beliefs from the previous level, thus providing a controlled means of reintroducing conventions and pragmatic reasoning into the learning process. Crucially, OBL removes the `weirdness' of learning in Dec-POMDPs, since, given the beliefs induced by the level below, each level has a unique optimal policy. Therefore, OBL at convergence can solve instances of the ZSC problem, as we illustrate in Hanabi. Importantly, OBL's performance gains under ZSC directly translate to state-of-the-art ad-hoc teamplay and human-AI coordination results, validating the ``ZSC hypothesis''.

An interesting direction for future work is to extend OBL to two-player zero-sum and general-sum settings, where the optimal \emph{grounded policy} might lead to a new type of equilibrium concept. 
Furthermore, it should be possible to adapt OBL to ensure convergence to a Nash equilibrium in these settings, e.g. by training \emph{average beliefs} and sampling the fictitious partner move randomly from levels below.  

Lastly, there are major features of human-level coordination that OBL doesn't capture: For example, human discarding behavior in Hanabi is based on being maximally predictable and robust to uncertainty in the beliefs.
More generally speaking, since episodes in ZSC are temporally extended, humans are commonly able to adapt to each other's policy and to introduce conventions \emph{within the episode}. OBL clearly highlights the need for fundamentally new methods that are capable of accomplishing this. 





\section{Acknowledgements}
We thank Andrei Lupu, Samuel Sokota, Michael Dennis and Wendelin Boehmer for their feedback on the manuscript. Lastly, we thank Board Game Arena for providing us with anonymized Hanabi game play data for this research and James Nesta (creator of hanabi.live) for providing pointers and advice.

\bibliography{reference}
\bibliographystyle{icml2021}

\appendix

\onecolumn

\section{Proofs of Theorems}
\label{app:proofs}

We first prove two lemmas that will be necessary for the proofs of the main theorems.

\begin{lemma}
For any policies ($\pi_0, \pi_1, \pi$),
\begin{multline}
    \sum_{\tau_t\in T_t} P(\tau_t|\pi_0) \sum_{a_t} \pi(a_t|\tau_t^i)\left( R(s_t, a_t) + \sum_{s_{t+1}} \mathcal{T}(s_{t+1}|s_t, a_t) V^{\pi_1}(\tau_{t+1}) \right) \\
    = \sum_{\tau_t^i\in T_t^i} P(\tau_t^i|\pi_0) \sum_{a_t} \pi(a_t|\tau_t^i) Q^{\pi_0 \to \pi_1}(a_t|\tau_t^i)
\end{multline}
\label{lemma:tau}
\end{lemma}
This Lemma shows an equivalence between an expected value integrated over all \textit{trajectories} of length $t$, and expected value integrated over all \textit{AOHs} of length $t$.
\begin{proof}
\begin{align}
&\sum_{\tau_t\in T_t} P(\tau_t|\pi_0) \sum_{a_t} \pi(a_t|\tau_t^i) \left( R(s_t, a_t) + \sum_{s_{t+1}} \mathcal{T}(s_{t+1}|s_t, a_t) V^{\pi_1}(\tau_{t+1}) \right)\\ 
&=\sum_{\tau_t^i\in T_t^i} P(\tau_t^i|\pi_0) \sum_{\tau_t\in \tau_t^i}  P(\tau_t|\tau_t^i) \sum_{a_t} \pi(a_t|\tau_t^i) \left( R(s_t, a_t) + \sum_{s_{t+1}} \mathcal{T}(\tau_{t+1}|\tau_t, a_t) V^{\pi_1}(\tau_{t+1}) \right) \\ 
&=\sum_{\tau_t^i\in T_t^i} P(\tau_t^i|\pi_0) \sum_{a_t} \pi(a_t|\tau_t^i)  \sum_{\tau_t\in \tau_t^i}  P(\tau_t|\tau_t^i)  \left( R(s_t, a_t) + \sum_{s_{t+1}} \mathcal{T}(s_{t+1}|s_t, a_t) V^{\pi_1}(\tau_{t+1}) \right) \\ 
&=\sum_{\tau_t^i\in T_t^i} P(\tau_t^i|\pi_0) \sum_{a_t} \pi(a_t|\tau_t^i) Q^{\pi_0\to \pi_1}(a_t|\tau_t^i) 
\end{align}
\end{proof}

\begin{lemma}
\label{lemma:softmax}
The softmax policy with temperature $T$ is worse than the optimal policy by at most $T/e$. Formally, for any $x_1 \ldots x_N \in \mathbb{R}^N$, 
\begin{equation}
\label{eq:softmax}
 \frac{ \sum_{i=1}^N \exp (x_i/T) x_i }{\sum_{j=1}^N \exp(x_j/T) \hspace{3mm}}  \geq \max\limits_i x_i - T/e
\end{equation}
\end{lemma}
\begin{proof}
Let $x_1 \geq x_2 \geq \ldots \geq x_N$ w.l.o.g. So $\max\limits_i x_i = x_1$. Let $y_i = x_i - x_1$.

\begin{align}
    \frac{ \sum_{i=1}^N \exp (x_i/T) x_i }{\sum_{j=1}^N \exp(x_j/T) \hspace{3mm}} & = \frac{ \sum_{i=1}^N \exp(x_1/T) \exp(y_i/T) (y_i + x_1) }{\sum_{j=1}^N \exp(x_1/T) \exp(y_j/T) \hspace{13mm}} &\\
    & = x_1 + \frac{ \sum_{i=1}^N \exp (y_i/T) y_i }{\sum_{j=1}^N \exp(y_j/T)} &\\
    & \geq x_1 + \frac{ \sum_{i=1}^N \exp (y_i/T) y_i }{N} & \text{since } y\leq 0 \\
    & \geq x_1 + \min\limits_{z \leq 0} \exp(z/T) z \label{eq:lemma1}
\end{align}

To compute the minimum of $\exp(z/T) z$ in $(-\infty, 0]$, 

\begin{align}
    &f(z) = \exp(z/T)z \\
    &\frac{df}{dz} = \exp(z/T)(1 + z/T) = 0 \\
    &z = -T \\
    &f(z) = -\exp(-1)T = -T/e
\end{align}

The value of $f(z)$ at both endpoints is 0. Therefore, $\min\limits_{z\leq 0} \exp(z/T)z = - T/e$, which when substituted into Equation \ref{eq:lemma1} proves the theorem.

\end{proof}

\thmunique*
\begin{proof}
\label{app:th1}
Since AOHs cannot repeat (a successor AOH is always longer than its predecessor) and the game has bounded length, the AOHs of acting players form a DAG with edges from each AOH to all possible successor AOHs for the next acting player. Ordering AOHs topologically s.t. successor AOHs precede their predecessors, we prove by induction.

In the base case, the first AOH in the topological ordering is a terminal AOH $\tau_\epsilon$, so $Q^{\pi_0\to \pi_1}(\cdot|\tau_\epsilon)=0$.

For the inductive case, we must show that if $Q^{\pi_0\to \pi_1}(a|\tau_{t+1}^j)$ is unique for every $\tau_{t+1}^j$ that is a successor of $\tau_t^i$, then $Q^{\pi_0\to\pi_1}(a|\tau_t^i)$ is unique (where $j$ is the player to act at the successor trajectory).

$\pi_1(a|\tau_t^i)$ is a function of $Q^{\pi_0 \to \pi_1}(\cdot|\tau_t^i)$.
\begin{equation}
    Q^{\pi_0 \to \pi_1}(a|\tau_t^i) = \sum_{\tau_t,\tau_{t+1}} R(s_t, a) + \mathcal{B}_{\pi_0}(\tau_t|\tau_t^i) \mathcal{T}(\tau_{t+1}|\tau_t) V^{\pi_0 \to \pi_1}(\tau_{t+1}^j)
    \label{eq:Q_unique}
\end{equation}
The RHS of Eq. \ref{eq:Q_unique} only depends on $\pi_1$ at successor AOHs, which are uniquely defined.
\end{proof}

\thmpolicyimprovement*

\begin{proof}
\label{app:th2}
Let $J(\pi_0 \xrightarrow{t} \pi_1)$ be the expected return of playing $\pi_0$ for the first $t - 1$ timesteps and $\pi_1$ subsequently.

We prove by induction backwards in $t$ that
\begin{equation}
    \label{eq:pi_induction}
    J(\pi_0 \xrightarrow{t} \pi_1) \geq J(\pi_0) - T(t_{max} - t)/e
\end{equation}
for all $t$, which implies the theorem for $t=0$.

The base case of $t=t_{max}$ is trivially true because $(\pi_0 \xrightarrow{t_{max}} \pi_1) \equiv \pi_0$.

For the inductive case, suppose that Eq. (\ref{eq:pi_induction}) holds for all $t > t'$. Let $T_t$ and $T_t^i$ be the set of all trajectories and AOHs of length $t$, respectively.

\small
\begin{align}
    J(\pi_0)-e(t_{max} - (t' + 1))T \leq J(\pi_0 \xrightarrow{t'+1} \pi_1) & \hspace{6.2cm}
\end{align}
\vspace{-0.5cm}
\begin{align}
    J(\pi_0 \xrightarrow{t'+1} \pi_1) & = \sum_{\tau_{t'+1}\in T_{t'+1}} P(\tau_{t'+1}|\pi_0) V^{\pi_1}(\tau_{t'+1}) \\
    &=\sum_{\tau_{t'}\in T_{t'}} P(\tau_{t'}|\pi_0) \sum_{a_{t'}} \pi_0(a_{t'}|\tau_{t'}^i)\left( R(s_{t'}, a_{t'}) + \sum_{s_{t'+1}}\mathcal{T}(s_{t'+1}|s_{t'}, a_{t'}) V^{\pi_1}(\tau_{t'+1})\right) \\ 
    &=\sum_{\tau_{t'}^i\in T_{t'}^i} P(\tau_{t'}^i|\pi_0) \sum_{a_{t'}} \pi_0(a_{t'}|\tau_{t'}^i) Q'_{\pi_0}(a_{t'}|\tau_{t'}^i) \hspace{1cm} \textrm{(Lemma \ref{lemma:tau})}\\ 
    & \leq \sum_{\tau_{t'}^i\in T_{t'}^i} P(\tau_{t'}^i|\pi_0) \max\limits_{a_{t'}}  Q'_{\pi_0}(a_{t'}|\tau_{t'}^i) \\ 
    &\leq\sum_{\tau_{t'}^i\in T_{t'}^i} P(\tau_{t'}^i|\pi_0) \sum_{a_{t'}} \pi_1(a_{t'}|\tau_{t'}^i) Q'_{\pi_0}(a_{t'}|\tau_{t'}^i) + T/e \hspace{1cm}\textrm{(Lemma \ref{lemma:softmax})}\\ 
    &=\sum_{\tau_{t'}\in T_{t'}} P(\tau_{t'}|\pi_0) \sum_{a_{t'}} \pi_1(a_{t'}|\tau_{t'}^i)\left( R(s_{t'},a_{t'}) + \sum_{s_{t'+1}} \mathcal{T}(\tau_{t'+1}|\tau_{t'}, a_{t'}) V^{\pi_1}(\tau_{t'+1})\right) + T/e  \hspace{0.1cm} \textrm{(Lemma \ref{lemma:tau})}\\ 
    &=\sum_{\tau_{t'}\in T_{t'}} P(\tau_{t'}|\pi_0) V^{\pi_1}(\tau_{t'}) + T/e \\ 
    &= \mathbb{E}_{T_{t'}}\left[V^{\pi_1}(\tau_{t'})|\pi_0\right] + T/e \\
    &= J(\pi_0 \xrightarrow{t'} \pi_1) + T/e
\end{align}
\normalsize

Lemmas \ref{lemma:tau} and \ref{lemma:softmax} can be found above.
\end{proof}

Notably, unlike standard MARL, the fixed points of the OBL learning rule \textit{are not} guaranteed to be equilibria of the game.

\thmequilibrium*

\begin{proof}
We'll start with a proof sketch for a ``temperature 0" policy $\pi_1(\tau^i)=\argmax\limits_{a} Q^{\pi_0\to \pi_1}(a|\tau^i)$ and then deal with the softmax policy.

Suppose $\pi$ is a fixed point of the temperature-0 OBL operator. Then at every AOH $\tau^i$,
\begin{align}
    \pi(\tau^i) &=\argmax\limits_a Q^{\pi\to\pi}(a|\tau^i) \\
    &= \argmax\limits_a Q^{\pi}(a|\tau^i)
\end{align}
By the one-shot deviation principle~\cite{HENDON1996274}, $\pi$ must be a subgame-perfect equilibrium of $G$.

Now, we will need to reiterate the proof of the one-shot deviation principle in order to modify it for softmax policies.

Suppose that $\pi$ is a fixed point of the OBL operator at temperature $T$. We will show that for any $\pi'_i$, 
\begin{equation}
    V^{\pi}(\tau^i) \geq V^{\pi'_i, \pi_{-i}}(\tau^i) - T/e(|\tau^i|-t_{max}).
    \label{eq:oneshot_induction}
\end{equation}
which for $\tau^i=\emptyset$ reduces to $J(\pi) \geq J(\pi'_i, \pi_{-i}) - t_{max}T/e$, proving the theorem.

We prove by induction over $i$'s AOHs, treating successor AOHs before predecessors as before.

The base case at terminal AOHs holds trivially. Now suppose that (\ref{eq:oneshot_induction}) is true for all successors of $\tau^i$. We know from Lemma \ref{lemma:softmax} that
\begin{align}
    \sum_a \pi(\tau^i) Q^{\pi\to\pi}(a|\tau^i) \geq \max\limits_a Q^{\pi\to\pi}(a|\tau^i) - T/e \\
    V^\pi(\tau^i) = \sum_a \pi(\tau^i) Q^{\pi}(a|\tau^i) \geq \max\limits_a Q^{\pi}(a|\tau^i) - T/e \label{eq:v1}
\end{align}
Let $P(\tau^{i'}|\tau^i,a,\pi_{-i})$ be the probability that player $i$'s next AOH is $\tau^{i'}$ after playing $a$ at $\tau^i$, and assuming other players play $\pi_{-i}$.\footnote{Crucially, this transition probability is not dependent on \textbf{player $i$'s} policy to reach $\tau^i$, because AOH $\tau^i$ already specifies all of player $i$'s actions to reach $\tau^i$.} Then expanding out $Q^\pi$ in (\ref{eq:v1}) leads to
\begin{align}
    V^\pi(\tau^i) & \geq \max\limits_a \sum_{\tau^{i'}} P(\tau^{i'}|\tau^i, a, \pi_{-i}) V^\pi(\tau^{i'}) - T/e \label{eq:q_expansion} \\
    & \geq \max\limits_a \sum_{\tau^{i'}} P(\tau^{i'}|\tau^i, a, \pi_{-i}) \left( V^{\pi'_i, \pi_{-i}}(\tau^{i'}) - T/e - T/e(|\tau^{i'}|-t_{max}) \right) \\
    & \geq \max\limits_a \sum_{\tau^{i'}} P(\tau^{i'}|\tau^i, a, \pi_{-i}) \left( V^{\pi'_i, \pi_{-i}}(a|\tau^{i'}) - T/e(|\tau^i|-t_{max}) \right) \\
    & \geq V^{\pi'_i, \pi_{-i}}(\tau^i) - T/e(|\tau^i| - t_{max})
\end{align}
Eq. (\ref{eq:q_expansion}) just expands out the expectation for $Q^\pi$ over all possible next AOHs that player $i$ may reach given the other players' policies $\pi_{-i}$.

\end{proof}

\thmgrounded*

\begin{proof}
We need only show that any policy $\pi_0$ that conditions only on public state yields a state distribution at each AOH that matches the grounded beliefs $\mathcal{B}_G$; then the optimal grounded policy follows directly from the definition of OBL (Eq. \ref{eq:pi1})

The true state distribution induced by $\pi_0$ at $\tau^i$ are
\begin{align}
    P(\tau|\tau^i, \pi_0) = \frac{P(\tau) \prod_t{ P( o_t^i | \tau) } \pi_0( a_t | \tau_t^{-i})}{\sum_{\tau'} P(\tau') \prod_t{ P( o_t^i | \tau')} \pi_0( a_t | \tau_t'^{-i})} 
\end{align}

If $\pi_0$ is constant, i.e. $\pi_0(a_t|\tau^i)=f(a_t)$, then $\pi_0$ in the numerator and denominator immediately cancel, yielding the grounded beliefs
\begin{align}
    P(\tau|\tau^i) = \frac{P(\tau) \prod_t{ P( o_t^i | \tau) }} {\sum_{\tau'} P(\tau') \prod_t{ P( o_t^i | \tau')}}.
\end{align}

If $\pi$ only depends on the public state, then the numerator and denominator still cancel for all trajectories $\tau$ which contribute to the sum, since any $\tau$ that doesn't share a common public state with $\tau'$ will lead to different observations, by definition.
\end{proof}
\section{Experimental Details for Hanabi}
\label{app:details}
\subsection{Reinforcement Learning}

\begin{figure}[h]
\centering
\includegraphics[width=0.75\linewidth]{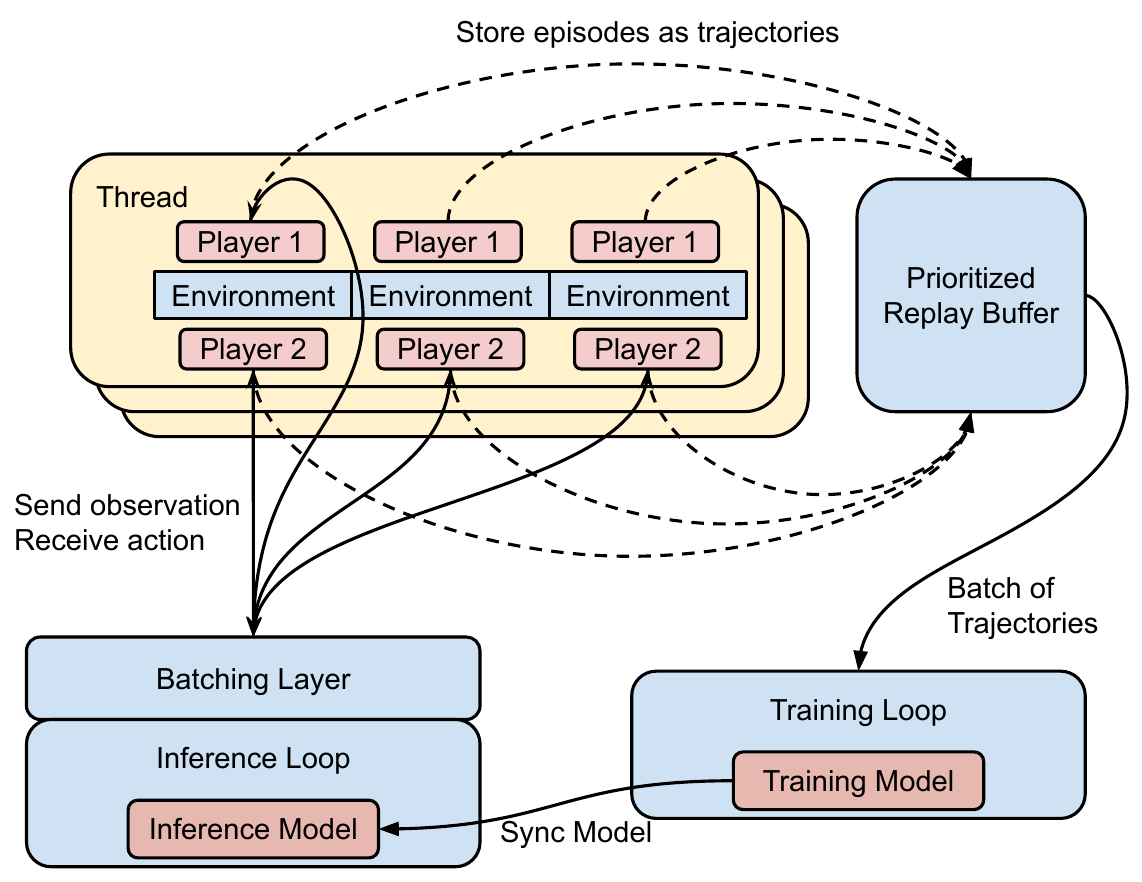}
\caption{\small Illustration of RL training setup. Some arrows linking \emph{player 1} and \emph{batching layer} are omitted for legibility.}
\label{fig:rl-stepup}
\end{figure}

We use a highly scalable setup as illustrated in Figure~\ref{fig:rl-stepup} to efficiently train RL models in Hanabi using moderate computational resources (three GPUs). It includes four major components. The first one is a large number of parallel thread workers that handle the interactions between environments and the multiple players in each environment. The player needs to invoke one or more neural network inferences at each step to compute an action, which we run on GPUs. Here, we take an asynchronous approach that instead of waiting for the neural network calls to return, the thread worker immediately moves on to execute the next set of environment and players after sending the neural network request to the inference loop. This allows us to run multiple environments in a single thread worker. On top of this, we also run multiple thread workers in parallel. The combination of these two techniques allows us to run a massive amount of environments and therefore generates a large amount of neural network inference requests simultaneously. These requests are then batched together before sending to an everlasting inference loop. The batched inference is executed on GPUs and the inference loop may use multiple GPUs to distribute workloads. Each player collects observations, actions and rewards at each step and aggregates them into a trajectory at the end of an episode. The trajectory is padded to a fixed length of 80 time-steps and then stored into the prioritized replay buffer. A training loop, which runs completely in parallel to all the procedures mentioned above, keeps sampling batches from the replay buffer and updates the model with TD-error. The training loop sends a new copy of the model to the inference loop to have the inference model synced every 10 gradient steps.
We also follow best Q-learning practices such as dueling architecture~\cite{dueling-dqn}, double DQN~\cite{double-dqn} and prioritized experience replay~\cite{prioritized-replay} with priority of the trajectory computed the same way as~\cite{r2d2}. In all of our RL experiments, i.e. both training OBL agents and reproducing Other-Play agents, we run 80 thread workers and 80 environments in each thread worker on 40 CPU cores. We use 2 GPUs for the inference loop and 1 GPU for the training loop. For OBL, we use 1 or 2 additional GPUs for belief model inference. We list out important hyper-parameters in Table~\ref{tab:hparam-rl}. 

\begin{table}[h]
    \centering
    \begin{tabular}{l l}
    \toprule
    Hyper-parameters     &  Value \\
    \midrule
    \texttt{\# replay buffer related}     &  \\
    \texttt{burn\_in\_frames} & \texttt{10,000} \\
    \texttt{replay\_buffer\_size} & \texttt{100,000} \\
    \texttt{priority\_exponent} & \texttt{0.9} \\
    \texttt{priority\_weight} & \texttt{0.6} \\
    \texttt{max\_trajectory\_length} & \texttt{80} \\
    \midrule
    \texttt{\# optimization} & \\
    \texttt{optimizer} & \texttt{Adam~\citep{kingma2014adam}} \\
    \texttt{lr} & \texttt{6.25e-05} \\
    \texttt{eps} & \texttt{1.5e-05} \\
    \texttt{grad\_clip} & \texttt{5} \\
    \texttt{batchsize} & \texttt{128} \\
    \midrule
    \texttt{\# Q learning} & \\
    \texttt{n\_step} & \texttt{1 (OBL), 3 (non-OBL)} \\
    \texttt{discount\_factor} & \texttt{0.999} \\
    \texttt{target\_network\_sync\_interval} & \texttt{2500} \\
    \texttt{exploration $\epsilon$} & $\epsilon_0 \dots \epsilon_n$, where $\epsilon_i = 0.1^{1 + {7i}/{(n-1)}} , n=80$ \\
    \bottomrule
    \end{tabular}
    \caption{\small Hyper-Parameters for Reinforcement Learning}
    \label{tab:hparam-rl}
\end{table}

\begin{table}[]
\centering
\begin{tabular}{l c c c c}
    \toprule
    \multirow{2}{*}{Agent} & \multicolumn{2}{c}{2 Players} & \multicolumn{2}{c}{3 Players} \\
    & Mean & Max & Mean & Max \\
    \midrule
    VDN~\cite{hu2019simplified} & 23.83 $\pm$ 0.03 & 23.96 & 23.71 $\pm$ 0.06 &  23.99 \\
    SAD~\cite{hu2019simplified} & 23.87 $\pm$ 0.03 & 24.01 & 23.69 $\pm$ 0.05 &  23.93 \\
    SAD+AUX~\cite{hu2019simplified} & 24.02 $\pm$ 0.01 & 24.08 & 23.56 $\pm$ 0.07 & 23.81 \\
    \midrule
    VDN (ours) & \textbf{24.27 $\pm$ 0.01} & \textbf{24.33} & \textbf{24.32 $\pm$ 0.02} & 24.39 \\
    SAD (ours) & 24.26 $\pm$ 0.01 & \textbf{24.33} & 24.24 $\pm$ 0.03 & \textbf{24.42} \\
    VDN+AUX (ours) & 24.26 $\pm$ 0.01 &  24.32 & 24.08 $\pm$ 0.01 & 24.17 \\
    \bottomrule
\end{tabular}
\caption{\small Self-play of various methods rerun using our infrastructure, comparing with previous published best results. For each method, we run 10 independent training with different seeds and shown the mean and max of the final model evaluated on 10K games. VDN is value-decomposition network that construct the joint Q-value with the sum of Q-values of each player. SAD is the simplified action decoder that includes additional representations for the greedy action beside the actual action that chosen by the agent via $\epsilon$-greedy or sampling from a softmax function. AUX means the addition of an auxiliary task that predict whether each card  is playable/discardable/unknown.}
\label{tab:hanabi-sp}
\end{table}

Notably, our implementation runs much faster in terms of both data generation due to a more efficient parallelization and thus can train on more data given a fixed, reasonable amount of time. After training for 40 hours, we are able to significantly outperform the state-of-the-art self-play performance in Hanabi using existing methods. As shown in Table~\ref{tab:hanabi-sp}, a simple VDN baseline is sufficient to achieve a remarkable 24.27 and 24.32 on average for 2-player and 3-player Hanabi, reducing the distance to perfect (25 points) by 26\% and 47\% respectively. We also find that with abundant data and fast data generation enabled by our infrastructure, the VDN baseline is strong enough that additional methods such as simplified action decoder (SAD) or auxiliary task (AUX) are no longer helpful for further improving self-play scores. However, we note that AUX still improves cross-play scores significantly.

\subsection{Belief Learning}
\label{app:belief}

Instead of creating a training/validation dataset, we use a similar setting as the reinforcement learning training, where mini-batches are sampled from an ever-changing replay buffer populated by thread workers. There are four main differences. First, the inference loop uses a pretrained fixed policy instead of syncing with a training policy periodically. Second, we store the true hand of the player alongside the trajectories which will be used as training targets for our belief model. Third, a normal experience replay buffer without priority is used. Finally, the training loop trains an auto-regressive model that predicts cards in hand one-by-one from the oldest to the latest with supervised learning. The network architecture is the same as the one used in ~\cite{hu2021learned}.

\subsection{Other Models}
\label{app:otherbots}
To test the performance of OBL when paired with unseen partners (\emph{ad-hoc} teamplay), we train three types of policies with existing methods. 

\textbf{Rank Bot.} The first is trained with the Other-Play (OP)~\cite{hu2020other} technique. The majority of policies trained with OP use a rank-based convention where they predominantly hint for ranks to indicate a playable card. For example, in a case where ``Red 1" and ``Red 2'' have been played and the partner just draw a new ``Red 3'', the other agent will hint 3 and then partner will play that card deeming that 3 being a red card based on that convention. We refer to this category of agents as ``Rank Bot".

\textbf{Color Bot.} Equivalently, one can also expect a color-based policy that will hint red for that latest card instead. In fact, such color-based policy was also appeared in the original Other-Play paper as the worst partner of the reset of the Other-Play policies. However, they are generally hard to reproduce as only 1 out of the 12 their training runs ended up with such policy. We use a simple reward shaping technique to produce similar policies reliably where we give each ``hint color" move an extra reward during the first half of the training process and then disable the extra reward in the second half to wash out any artifacts. However, the reward shaping introduces undesired side effects that lead to arbitrary conventions. In practice, we find that hiding the \emph{last action} field of the input observation will make training outcomes more consistent. The reward shaping and feature engineering techniques are only used for producing this specific policy, which we refer to as ``Color Bot".

\textbf{Clone Bot.} One of the goals of the Hanabi challenge~\citep{bard2020hanabi} is to develop artificial agents that can collaborate with humans. To understand how well our models can collaborate with humans without resorting to costly experiments, we train a behavior clone bot mimicking human behaviors to serve as a proxy. We acquire 240,954 2-player games and 113,900 3-player games from the online game platform ``Board Game Arena"\footnote{\url{https://en.boardgamearena.com/}} and convert it to a dataset for supervised learning. The model takes in the trajectory of an entire game from the perspective of a single player and predicts its action at each time-step. This is similar to the independent Q-learning setting in the multi-agent RL context. The network consists of one fully layer, followed by ReLU activation, a two-layer LSTM and an output fully connected layer with softmax. Dropout is applied before the output layer to reduce overfitting. The best models use 512 hidden units for each layer in the network and 0.5 dropout rate. We adapt the color shuffling technique of Other-Play~\cite{hu2020other} as a data augmentation tool. At each training step, we randomly shuffle the color space for both observation and action of each trajectory in the mini-batch before feeding them to the network. We find this technique significantly improves the performance of the supervised model, especially for 3-player Hanabi where we have much less data. The results are shown in Table~\ref{tab:clone-bot}. The agent is trained by minimizing the cross-entropy loss. At test time, the agent selects the action with the highest probability at each step. We pick the model that achieves the highest self-play score during training as our final ``Clone Bot".

\begin{table}[]
    \centering
    \begin{tabular}{c c c}
    \toprule
         & 2 Players & 3 Players \\
    \midrule
         w/o Other-Play & 19.93 $\pm$ 0.09 & 13.07 $\pm$ 0.15 \\
         w Other-Play    & \textbf{21.01 $\pm$ 0.07} & \textbf{17.46 $\pm$ 0.12} \\
    \bottomrule
    \end{tabular}
    \caption{\small Results of Clone Bot. Here Other-Player refer to the color shuffling technique first proposed in the Other-Player paper. We adapt the same technique here but use it for data-augmentation purpose. The trained models are evaluated on 5000 games by greedily selecting the action of the highest probability at test time. Errors shown are standard error of mean (s.e.m.)}
    \label{tab:clone-bot}
\end{table}


\section{Qualitative Analysis of Learned Hanabi Agents}
\label{app:quality}

In this section we share more insights about the policies learned by different methods through a series of qualitative analysis. 

As we can see from Figure~\ref{fig:hanabi-card-knowlege} in the main paper, the OBL level 1 agent plays a significantly higher percentage of grounded cards, \emph{i.e.} cards of which both color and rank are known. This supports our claim that OBL learns a grounded policy given a constant $\pi_0$. However, it is worth mentioning that OBL level 1 still plays a fair number of cards when it only knows their rank. At first this may seem contradictory to the grounded policy claim. However, as we analyze the games, we find that in many cases the best grounded policy is to play those cards even if the color is not known. For example, at an early stage of the game when only a few cards have been played, knowing a card is 1 is sufficient to play it. Moreover, the agent often knows that the card is of a ``safe color". For example, when both red 2 and green 2 have been played, it will be safe to play a card if we know that it is either a red 3 or a green 3. Playing such a card will be classified as ``only rank" in the figure. Furthermore, there are multiple lives in Hanabi and sometimes it is worth taking a risk if the odds are good. We observe that the OBL level 1 agent sometimes plays a card blindly without knowing any information about the card, as reflected in the figure by the slightly higher percentage of ``none" category than other OBL levels and other agents except for SAD. This is because Hanabi is not designed to be a game that can be mastered by a grounded policy, and therefore the number of the remaining hint tokens are often low due to the aggressive hinting strategy used by the policy. In this case, OBL level 1 resorts to maximizing the utility of the life tokens by playing one or two cards blindly in later stages of the game when the chances that a newly drawn card could be useful is high and remaining life tokens are abundant. Once there is only one life token left, the agent becomes conservative and stops the gambling behavior.

The conventions in the high levels of OBL agents stem from the fact that OBL level 1 agents give hints about a card in certain orders based on the situations. In some cases the OBL level 1 agent first hints at the color of a card it wants the partner to play while in other cases it hints rank first. The order depends on the grounded probability of having a playable card in the partner's hand after giving the hint assuming only the public information. If the partner knowing the color of some cards leads to a higher chance that they will have a playable card in hand, then we should hint color first and vice versa. This information gets picked up by the belief model and then taken advantage of by the OBL in the next level to form conventions. The conventions gradually get reinforced and lead to more nuanced conventions in the higher level of OBL agents.

For reference, we can also see from the same Figure~\ref{fig:hanabi-card-knowlege} that the Other-Play Rank Bot and Color Bot have strong preferences to use only rank and only color to exchange information respectively, making it difficult for them to coordinate with each other. The Clone Bot uses both color and rank depending on the situation, which is similar to the highest level OBL agent. The SAD agent, which is trained in the typical self-play setting, seems to play ``blindly" a lot. However, the fact that such an agent achieves a high self-play score indicates that it must be using some form of secretive conventions that are not grounded at all, e.g. conventions such as ``hinting red means play the second card", which completely disregard the grounded information revealed by the hint actions. Since these conventions assign arbitrary and unpredictable meanings to actions that vary between every independent training run, they fail under the zero-shot coordination setting and are similarly poor for human-AI coordination.

\section{Off-Belief Learning on 3 Player Hanabi}
\label{app:3player}

\begin{table}[h]
    \centering
    \begin{tabular}{l c c c c}
    \toprule
    Method & Self-Play & Cross-Play & w/ Clone Bot \\
    \midrule
    Other-Play    & 23.98 $\pm$ 0.03  & 17.36 $\pm$ 0.19 &  12.18 $\pm$ 0.25 \\
    \midrule
    OBL (level 1) & 20.52 $\pm$ 0.05 & 20.41 $\pm$ 0.01 & 13.10 $\pm$ 0.42 \\
    OBL (level 2) & 22.71 $\pm$ 0.04 & 22.50 $\pm$ 0.01 & 14.09 $\pm$ 0.48\\
    OBL (level 3) & 23.19 $\pm$ 0.03 & 22.85 $\pm$ 0.01 & 13.96 $\pm$ 0.47\\
    OBL (level 4) & 23.38 $\pm$ 0.04 & 23.02 $\pm$ 0.01 & 13.88 $\pm$ 0.47   \\
    \bottomrule
    \end{tabular}
    \caption{Performance in 3 player Hanabi. We run 10 independent training with different seeds for each method. Self-Play indicates play between an agent and itself. Cross-Play indicates play between agents from different independently-trained runs of the same algorithm. For 3 player Hanabi, the Cross-Play is computed by pairing agents from 3 different training runs. When evaluating agents with clone bot, we average the results of 2 copies of the same agents with 1 clone bot and 1 agent with 2 copies of the same clone bots. Each combination of agents are evaluated on 5000 games.}
    \label{tab:tab:hanabi-3p-result}
\end{table}

To better demonstrate the generalizability of off-belief learning. we apply it to 3-player Hanabi. The results are shown in Table~\ref{tab:tab:hanabi-3p-result}. The cross-Play between 3 players can either consist of two copies of one agent and one copy of another or three different agents. We find the latter to be more challenging and therefore the cross-Play in the table is computed that way, except for Clone Bot. The score with Clone Bot is the average of two Clone Bots with one other bot and one Clone Bot with two other bots. We run 10 independent training seeds for each method. Despite high Self-Play score, Other-Play agents suffer greatly in the cross-Play, even more than they do in 2 Player Hanabi. However, the gap between cross-Play and Self-Play scores for each OBL level are significantly smaller. Therefore, even the OBL level 1 agent that does not use any conventions is able to collaborate better in the zero-shot coordination setting. As we apply OBL for more levels, the cross-Play score grows on the same pace as Self-Play score. Notably, all OBL levels perform better than Other-Play when collaborating with Clone Bot, indicating OBL as a promising method for human-AI coordination.
\end{document}